\documentclass[runningheads]{llncs}
\usepackage[T1]{fontenc}
\usepackage{graphicx}
\usepackage{amsmath,amssymb}
\usepackage{xcolor}
\usepackage{tikz}
\usepackage{subcaption}
\usepackage{algorithm}
\usepackage{algpseudocode}
\usepackage{gensymb} 
\usepackage[bookmarks=true]{hyperref}

\usetikzlibrary{arrows.meta}

\definecolor{myPurple}{RGB}{160,32,240}
\definecolor{mygrey}{RGB}{169,169,169}
\definecolor{darkgreen}{RGB}{0,100,0}
\definecolor{dark-red}{rgb}{0.4,0.15,0.15}
\definecolor{dark-blue}{rgb}{0.15,0.15,0.6}
\definecolor{medium-blue}{rgb}{0,0,0.5}
\definecolor{ForestGreen}{RGB}{34,139,34}

\newcommand{\R}{\mathbb{R}}

\newcommand{\na}{{\scriptsize N/A}}

\spnewtheorem{requirement}{Requirement}{\bfseries}{\itshape}
\spnewtheorem{assumption}{Assumption}{\bfseries}{\itshape}
\spnewtheorem{notation}{Notation}{\bfseries}{\itshape}

\renewenvironment{proof}[1][Proof]{\par\noindent\textit{#1.}\enspace}{\qed\par}

\long\def\invis#1{}

\begin{document}
\title{On the Control of Mobile Ad-Hoc Agent Deployments in Partially Observed Space}
%
\titlerunning{Control of Mobile Ad-Hoc Agent Deployments}
%
\author{Edwin Meriaux\inst{1,3}$^{,\star}$ \and
Louis-Roy Langevin\inst{2}$^{,\star}$ \and
Shuo Wen\inst{1} \and
Ndiam\'e Ndiaye\inst{2} \and
Gregory Dudek\inst{1} \and
Antonio Lor\'{\i}a\inst{3}}
\authorrunning{E. Meriaux, L.-R. Langevin et al.}
\institute{Computer Science Department and Mila, McGill University, Montreal, Canada \and
Department of Mathematics and Statistics, McGill University, Montreal, Canada \and
L2S, CNRS--CentraleSup\'elec, Universit\'e Paris-Saclay, Gif-sur-Yvette, France\\[4pt]
$^{\star}$~Edwin Meriaux and Louis-Roy Langevin are co-first authors.}
\maketitle              
{
  \renewcommand{\thefootnote}{}%
  
  \addtocounter{footnote}{-1}%
}
\begin{abstract}
We study the online deployment of mobile ad hoc networks in unknown orthogonal
environments, formalized as the Partially Observable Cooperative Guard Art
Gallery Problem. We give a full proof that CADENCE algorithms achieve full coverage while maintaining a
connected visibility graph using at most $\tfrac{n}{2}+h-2$ agents in orthogonal worlds
with $n$ corners and $h$ holes. We further evaluate deployment-order heuristics that reduce agent count and deployment time in practice.

\end{abstract}
\section{Introduction}Mobile ad hoc networks~\cite{anjum2015survey} (MANETs) are teams of mobile
agents that establish communication and sensing coverage without
pre-existing infrastructure or knowledge of the environment. They arise in
indoor, underground, and confined spaces such as disaster-stricken buildings,
tunnels, and mines, where communication infrastructure is unavailable,
unreliable, or prohibitively expensive to
deploy~\cite{anjum2015survey,kulla2015real,del2016survey,8991194,zhang2023unmanned},
and where obstacles and irregular geometry force agents to organize under
limited range and line-of-sight conditions~\cite{nirmaladevi2022augmented}. 
The combinatorial core of such deployments, covering an environment with
line-of-sight agents, is the \emph{Cooperative Guard Art Gallery Problem}~\cite{zylinski2004cooperative}. When
the environment is discovered online rather than known in advance, it becomes
the \emph{Partially Observable Cooperative Guard Art Gallery Problem} (POCGAGP),
introduced in~\cite{meriaux2025mobile}: agents enter an unknown environment,
expand coverage monotonically, and preserve connectivity throughout.

Rather than a single algorithm, we present \emph{World Coverage Algorithms} (WCA), which for a given world $W$ (see Definition~\ref{def:world}), solves the POCGAGP (see Requirements~\ref{req:conn},\ref{req:cov},\ref{req:budget},\ref{req:time}) with its requirements of constant connectivity between the deployed agents, monotonically increasing coverage, and limiting its usage of resources (the number of agents and time to completion). This represents the deployment of a MANET to form a connected network over the whole space. WCA
deploy agents onto points visible to the current network and may relocate them
to adjacent visible points; both preserve connectivity, so on completion the
agents cover $W$ as one connected network. Members differ in their rules, and
may reach different covering configurations.

A key subclass of this algorithm class is the CADENCE (Centralized
Agent Directed Exploration for Network Coverage)~\cite{meriaux2025mobile} Algorithm which
deploys agents to \emph{valid $270^\circ$ corners} of an
orthogonal world, yielding the worst-case bound $\tfrac{n}{2}+h-2$ for $n$
corners and $h$ holes, matching the best known bound for the Cooperative Guard Art Gallery
Problem~\cite{zylinski2004cooperative} under partial observability.

We make three contributions. First, we introduce the class of World Coverage Algorithms, an exact classification of the algorithms that solve the POCGAGP. Second, we present a full proof of the $\tfrac{n}{2}+h-2$ bound
for CADENCE, achieved despite the partial observability constraint (formalizing a proof already sketched in a previous work~\cite{meriaux2025mobile}). Third, we report an empirical study of over $10{,}000$ tests
comparing members of CADENCE with different deployment-order heuristics.
\section{Mathematical Background}
\label{math}

\begin{definition}[Closure, Border, Interior]
    Given a set $S \subseteq \R^2$, the \textbf{closure} of $S$ is defined as 
    $\mathrm{cl}(S) := \{x \in \R^2 \,:\, \exists (x_i)_{i \ge 1} \text{ with } x_i \in S \text{ and } x_i \to x\}$, 
    the \textbf{border} of $S$ as $\partial S := \mathrm{cl}(S) \cap \mathrm{cl}(S^c)$, and 
    the \textbf{interior} of $S$ as $\mathrm{int}(S) := S\setminus\partial S$.
\end{definition}

\begin{definition}[Polygon, Edge, Interior, Corner, Orthogonal, Filled]
    A \textbf{polygon} with $n \ge 3$ distinct \textbf{corners} $c_1,\ldots,c_n \in \R^2$ consists of the union of the line segments $\overline{c_i c_{i+1}}$ for each $1 \le i \le n-1$ and $\overline{c_n c_{1}}$. These line segments are called the \textbf{edges} of the polygon and two edges can only intersect on a common endpoint (i.e. on a common corner). The set of points of an edge excluding its endpoints is called the \textbf{interior} of this edge. A polygon is called \textbf{orthogonal} if all its edges are horizontal or vertical with respect to the Cartesian axes.  A \textbf{filled polygon} is the bounded region formed by a polygon, including the border.
\end{definition}

\begin{remark}
    Note that every filled polygon has positive area, or equivalently non-empty interior. That is, polygons cannot degenerate into a line segment or point.
    \label{world_asp}
\end{remark}

\begin{definition}[World, Piece, Hole, Corner, Orthogonal]\label{def:world}
    A set $W \subseteq \R^2$ is called a \textbf{world} if there exist $h \ge 0$ filled polygons $P_0,P_1,\ldots,P_h$ such that $P_1,\ldots,P_h$ are pairwise disjoint and
    \[
        \bigcup_{i=1}^h P_i \subset \mathrm{int}(P_0) \quad \text{and} \quad W = P_0\setminus\Big(\bigcup_{i=1}^h \mathrm{int}(P_i)\Big).
    \]
    We call $P_0,\ldots,P_h$ the \textbf{pieces} of $W$ and $P_1,\ldots,P_h$ the \textbf{holes} of $W$.  A point $x \in W$ is called a \textbf{corner} of $W$ if it is a corner of one of $P_0,P_1,\ldots,P_h$. If all the pieces of $W$ are orthogonal, then we say that $W$ is \textbf{orthogonal}.
    \label{world}
\end{definition}

\begin{remark}
    Note that the pieces of a world are uniquely defined and that it follows from this definition that $\partial P_0,\ldots,\partial P_h$ are pairwise disjoint and that $\partial W = \partial P_0 \cup \partial P_1 \cup \ldots \cup \partial P_h$.
\end{remark}

\begin{definition}[Angle]\label{def:angle}
    For a set $S \subseteq \R^2$, $x \in S$, and $0 \le r \le 360$, $x$ is called an $r^\circ$ \textbf{angle} in $S$ if $\lim_{\varepsilon \downarrow 0} \lambda[B_\varepsilon(x)\cap S]/(\pi\varepsilon^2) = r/360$ if the limit exists where $\lambda$ denotes the area and
    $B_\varepsilon(x)$ is the ball of radius $\varepsilon$ centered at $x$.  We may simply say $r\degree$ angle if $S$ is not ambiguous.
\end{definition}
\noindent 

\begin{remark}
    If $W$ is an orthogonal world, simple analysis can show that any $x \in W$ is a $r\degree$ angle for some unique $r \in \{90,180,270,360\}$. Furthermore, a point $x \in W$ is a corner of $W$ if and only if it is a $90^\circ$ angle or a $270^\circ$ angle, in which case we may call $x$ a $90^\circ$ corner or a $270^\circ$ corner, respectively. A point $x \in W$ is an interior point of $W$ (i.e. $x \in \mathrm{int}(W)$) if and only if it is a $360^\circ$ angle. Finally, a point $x \in W$ lies on the interior of an edge of one of the pieces of $W$ if and only if $x$ is a $180^\circ$ angle.
\end{remark}

\begin{definition}[Invalid corner, Valid corner]
    For an orthogonal world $W$, the \textbf{invalid corner} of a hole $P_i$ is its  corner that has the largest $y$-coordinate, and with smallest $x$-coordinate to break equalities (top-left corner \cite{meriaux2025mobile}). The set of \textbf{valid corners} of an orthogonal world $W$, denoted $C_W$, is the set of all its $270^\circ$ angles that are not the invalid corners of holes.
    \label{def:valid}
\end{definition}

\begin{remark}
    If $W$ is an orthogonal world and $c$ is the invalid corner of a hole $P_i$, then $c$ must be a $270^\circ$ angle in $W$, hence all invalid corners of $W$ are $270^\circ$ angles.
\end{remark}

\begin{definition}[Connectivity]
\label{conn}
    A graph $G$ is said to be \textbf{connected} if any two vertices of this graph are linked by a path. Similarly, a set $S \subseteq \R^2$ is said to be \textbf{connected} if any two points in $S$ are linked by a continuous path contained inside of $S$.
    \label{def:connectivity}
\end{definition}

\begin{lemma}
\label{connlem}
    Let $G = (V,E)$ be a connected graph and $S \ne \emptyset$ be a proper subset of $V$. Then, there exist $u \in S$ and $v \notin S$ such that $uv \in E$\cite{west2001introduction}.
\end{lemma}

\begin{definition}[Field of view]
\label{FOV}
    Let $W \subseteq \R^2$ and $S$ be a finite subset of $W$. The \textbf{field of view} of $S$ in $W$ is defined as:
    \[
    F_{W}(S) = \bigcup_{s \in S} \big\{ x \in W :\ \overline{xs} \subset W\big\}.
    \]
    where $\overline{xs}$ is the line segment linking $x$ to $s$.
    We may simply write $F(S)$ if $ W$ is not ambiguous. We say that $S$ \emph{covers} $W$ if $F(S) = W$, that is, every point of $W$ is seen by some point in $S$.
\end{definition}

\begin{definition}[Visibility graph]\label{def:visibility-graph}
Let $W \subseteq \mathbb{R}^2$ and let $S$ be a finite subset of $W$. The
\textbf{visibility graph} of $S$ in $W$, denoted $G(S)$, has vertex set
$V(S) = S$ and edge set:
\[
  E(S) = \{\, xy : x, y \in S \text{ and } x \in F(\{y\}) \,\}.
\]

\label{V_G}
\end{definition}

\begin{remark} 
    It is well-established that if $S$ is the set of reflex corners (i.e. angles greater than $180^\circ$) of a non-convex world $W$ , then $G(S)$ is connected and $S$ covers $W$~\cite{de2000computational,dudek2024computational,kapoor1988efficient,latombe2012robot,lavalle2006planning,o2017visibility,tan2009visibility}
    \label{vg}
\end{remark}



\section{The Partially Observable Cooperative Guard Art Gallery Problem (POCGAGP)}
\label{sec:pocgagp}



Consider the real-world deployment of a continuously connected MANET in an environment containing walls and obstacles that is initially unknown and incrementally being discovered. As the coverage expands, with a bounded number of agents and after a bounded amount of time, the objective is that every part of the environment is observed by some agent. Furthermore, for any time stamp $t \ge 0$, if a part of the environment is seen at time $t$, then it must also be seen at time $t+1$. In this setting, there exists a \emph{deployment point} from which agents spawn, which acts as a node in the MANET and enables communication with the outside world. The MANET satisfies the following assumption:

\begin{assumption}[Mutual visibility]\label{asp1}
The members of the MANET are equipped with sensors such that any two members $x,y$ can directly communicate with each other if $x$ observes $y$ (and equivalently $y$ observes $x$). 
\label{mutual}
\end{assumption}

The objective of the POCGAGP is to create a geometric representation of this MANET deployment and design efficient algorithms to fully discover the environment in which they are deployed. We represent such environments as subsets $W$ of $\R^2$ (corresponding to worlds defined in the previous section), where the agents become points in $W$ and the deployment point is denoted as $x_d \in W$. If $\mathcal{A} \subset W$ represents our set of agents, then $G(\mathcal{A} \cup \{ x_d\})$ represents the MANET, where two members of the MANET communicating with each other are represented by points $x,y \in W$ being in each other's field of view ($x \in F_W(\{y\})$, or equivalently $y \in F_W(\{x\})$), corresponding to Assumption~\ref{mutual}.

To properly represent the real-world MANET deployment and its constraints, the following requirements are added to the POCGAGP:

\begin{requirement}[Constant connectivity]
At any point of the algorithm, the visibility graph $G(\mathcal{A} \cup \{ x_d\})$ formed by the deployed agents together with the deployment point must be connected.
\label{req:conn}
\end{requirement}

\begin{requirement}[Increasing coverage]
Coverage must evolve monotonically. That is, if $\mathcal{A}$ and $\mathcal{A}^+$ denote the positions of agents at consecutive time steps, then $F(\mathcal{A} \cup \{x_d\}) \subseteq F(\mathcal{A}^+ \cup \{x_d\})$.
\label{req:cov}
\end{requirement}

  \begin{requirement}[Agent budget]
Given a world $W$, the number of deployed agents never exceeds some number $N_{\max}$ (dependent on $W$), representing the resource constraints of the real-world deployment.
\label{req:budget}
\end{requirement}
\begin{requirement}[Time budget]
Given a world $W$, the deployment must terminate within a finite time horizon $T_{\max}$ (dependent on $W$), representing the time constraints of the real-world deployment.
\label{req:time}
\end{requirement}

\section{World Coverage Algorithms and CADENCE}
\subsection{The Class of World Coverage Algorithms}
To approach the POCGAGP, we make use of a certain type of algorithm that we call World Coverage Algorithms (WCA)\footnote{Note that this class changes depending on the values of $N_{\max}$ and $T_{\max}$ because of Req~\ref{req:budget} and Req~\ref{req:time}, respectively.}. Every such algorithm takes a world $W$ and a deployment point $x_d \in W$ as inputs, and can be uniquely defined by its transition rule $\tau$.

This rule $\tau$ takes three arguments: the deployment point $x_d$, a (finite) set $\mathcal{A} \subseteq W$ representing the agents currently allocated, and a (non-finite) set $\mathcal{F} \subset W$ representing the field of view of $\mathcal{A}\cup\{x_d\}$. Then, $\mathcal{A}^+ := \tau(x_d,\mathcal{A},\mathcal{F})$ will be the (finite) set of allocated agents in the next step of the algorithm. It must always hold that $\mathcal{A}^+ \subseteq \mathcal{F}$ to ensure that no agent is ever allocated to a location that has not been discovered yet.

To satisfy the four requirements stated in the previous section, $\tau$ must obey certain conditions. First, we require that $G(\mathcal{A}^+ \cup \{x_d\})$ is always connected~\eqref{req:conn}, and that $F(\mathcal{A} \cup \{x_d\}) \subseteq F(\mathcal{A}^+ \cup \{x_d\})$ always holds~\eqref{req:cov}. Furthermore, given an input $W$, we require that $|\mathcal{A}^+| \le N_{\max}$~\eqref{req:budget}, and that $\tau^{(T_{\max})}(x_d,\emptyset,F(x_d)) = W$, i.e. that after using the transition rule $\tau$ exactly $T_{\max}$ times starting from the initial state $(x_d,\emptyset,F(x_d))$, the world must be fully covered~\eqref{req:time}.

In order to properly model the fact that the agents do not know the shape of their environment before observing it, this function $\tau$ has to act independently from the input $W$. This way, all the decisions that are made by the algorithms of WCA are based uniquely on the location of the deployment point $x_d$, the current allocation $\mathcal{A}$ of the agents, and the set $\mathcal{F}$ of what is visible of $W$, and therefore cannot depend on parts of $W$ that have not been observed yet.

To help us understand how such algorithms are executed, we provide the following general template.

\begin{algorithm}
\caption*{\textbf{World Coverage Algorithm Template}}
\label{wca_algo}
    \begin{algorithmic}[1]
    \State \textbf{Input:} a world $W$ and a deployment point $x_d \in W$.
    \State $\mathcal{A} \gets \emptyset$;\quad $\mathcal{F} \gets F(\{x_d\})$
    \While{$\mathcal{F} \ne W$}
        \State $\mathcal{A}^+ \gets \tau(x_d,\mathcal{A}, \mathcal{F})$
        \State $\mathcal{A} \gets \mathcal{A}^+$, $\mathcal{F} \gets F(\mathcal{A}^+ \cup \{x_d\})$
    \EndWhile
    \State \textbf{Output:} $\mathcal{A}, \mathcal{F}$.
    \end{algorithmic}
\end{algorithm}
\vspace{-1cm}

\subsection{CADENCE}

We introduce the subclass of WCA that we call CADENCE (Algorithm~\ref{cadence_algo}) and that solves the POCGAGP efficiently in the orthogonal case. Given an orthogonal world $W$ and a deployment point $x_d \in W$ as inputs, its transition rule $\tau$ takes a state $(x_d,\mathcal{A},\mathcal{F})$, picks any visible valid corner $x$ that is not in $\mathcal{A}$ yet (i.e. $x \in \mathcal{F} \cap C_W \setminus \mathcal{A}$), and gives back $\mathcal{A}^+ = \mathcal{A}\cup\{x\}$ (i.e. deploys a new agent at $x$).

Notice that multiple distinct algorithms are part of CADENCE, since they differ by their choice of $x$. For example, one could pick $x$ randomly, while another could pick the $x$ that is the furthest to $x_d$. In Section~\ref{sec:experiments}, we compare such different ways of choosing $x$ and their efficiency in various environments.

\begin{algorithm}
\caption{CADENCE algorithms}
\label{cadence_algo}
    \begin{algorithmic}[1]
    \State \textbf{Input :} an orthogonal world $W$ and a deployment point $x_d \in W$.
    \State $\mathcal{A} \gets \emptyset$;\quad $\mathcal{F} \gets F(\{x_d\})$
    \While{$\mathcal{F} \ne W$}
    \State Pick $x \in \mathcal{F}\cap C_W\setminus \mathcal{A}$
    \State $\mathcal{A} \gets \mathcal{A}\cup\{x\}$;\quad $\mathcal{F} \gets \mathcal{F}\cup F(\{x\})$
    \EndWhile
    \State \textbf{Output :} $\mathcal{A}, \mathcal{F}$.
    \end{algorithmic}
\end{algorithm}
\setlength{\textfloatsep}{3pt}
\setlength{\intextsep}{3pt}

Theorem~\ref{thm1} gives us that $x_d \in F(C_W)$, guaranteeing the existence of $x$ for the first step and giving us some non-empty $\mathcal{A}^+$ in the first iteration of CADENCE. From this point on, since $G(C_W)$ is connected by Theorem~\ref{main2}, then as long as $\mathcal{A} \ne C_W$, there will always exist some visible valid corner $x$ that is not already in $\mathcal{A}$ by Lemma~\ref{connlem}. Sample deployments of CADENCE can be seen in Figure~\ref{fig:polygonal-guards}.

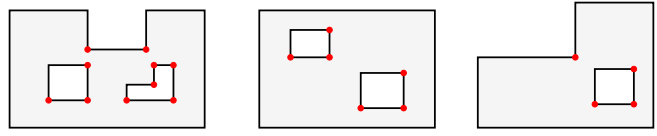
\begin{figure}[]
\centering
\resizebox{0.7\linewidth}{!}{%
\begin{tikzpicture}[scale=0.62, transform shape,
  wall/.style={thick},
  agent/.style={fill=red, draw=red, circle, inner sep=0pt, minimum size=4pt},
]
\begin{scope}[xshift=0cm]
  \fill[black!4]
    (0,0) -- (5,0) -- (5,3) -- (3.5,3) -- (3.5,2) -- (2,2) -- (2,3) -- (0,3) -- cycle;
  \draw[wall]
    (0,0) -- (5,0) -- (5,3) -- (3.5,3) -- (3.5,2) -- (2,2) -- (2,3) -- (0,3) -- cycle;
  \fill[white] (1,0.7) rectangle (2,1.6);
  \draw[wall] (1,0.7) rectangle (2,1.6);
  \fill[white] (3,0.7) -- (4.2,0.7) -- (4.2,1.6) -- (3.7,1.6) -- (3.7,1.1) -- (3,1.1) -- cycle;
  \draw[wall] (3,0.7) -- (4.2,0.7) -- (4.2,1.6) -- (3.7,1.6) -- (3.7,1.1) -- (3,1.1) -- cycle;
  \node[agent] at (3.5,2) {}; \node[agent] at (2,2) {};
  \node[agent] at (1,0.7) {}; \node[agent] at (2,0.7) {}; \node[agent] at (2,1.6) {};
  \node[agent] at (3,0.7) {}; \node[agent] at (4.2,0.7) {}; \node[agent] at (4.2,1.6) {};
  \node[agent] at (3.7,1.6) {}; \node[agent] at (3.7,1.1) {};
\end{scope}
\begin{scope}[xshift=6.4cm]
  \fill[black!4] (0,0) rectangle (4.5,3);
  \draw[wall] (0,0) rectangle (4.5,3);
  \fill[white] (0.8,1.8) rectangle (1.8,2.5);
  \draw[wall] (0.8,1.8) rectangle (1.8,2.5);
  \fill[white] (2.6,0.5) rectangle (3.7,1.4);
  \draw[wall] (2.6,0.5) rectangle (3.7,1.4);
  \node[agent] at (1.8,2.5) {}; \node[agent] at (0.8,1.8) {}; \node[agent] at (1.8,1.8) {};
  \node[agent] at (3.7,1.4) {}; \node[agent] at (2.6,0.5) {}; \node[agent] at (3.7,0.5) {};
\end{scope}
\begin{scope}[xshift=12.0cm]
  \fill[black!4]
    (0,0) -- (4.5,0) -- (4.5,3.2) -- (2.5,3.2) -- (2.5,1.8) -- (0,1.8) -- cycle;
  \draw[wall]
    (0,0) -- (4.5,0) -- (4.5,3.2) -- (2.5,3.2) -- (2.5,1.8) -- (0,1.8) -- cycle;
  \fill[white] (3,0.6) rectangle (4,1.5);
  \draw[wall] (3,0.6) rectangle (4,1.5);
  \node[agent] at (2.5,1.8) {};
  \node[agent] at (4,1.5) {}; \node[agent] at (3,0.6) {}; \node[agent] at (4,0.6) {};
\end{scope}
\end{tikzpicture}%
}
\caption{Sample agent placements (red) under CADENCE in orthogonal worlds.}
\label{fig:polygonal-guards}
\end{figure}

Note that as CADENCE algorithms run, the graph $G(\mathcal{A} \cup \{x_d\})$  remains connected and increases, meaning that $F(\mathcal{A} \cup \{x_d\})$ never decreases. Furthermore, since they add one element of $C_W$ to $\mathcal{A}$ at each step and $F(C_W) = W$ by Theorem~\ref{thm1}, they need at most $|C_W|$ iterations before $\mathcal{F}(\mathcal{A} \cup \{x_d\}) = W$. From Proposition~\ref{prop_bound}, this indicates that any CADENCE algorithm always halts after at most $\frac{n}{2}+h-2$ iterations, meaning that those algorithms satisfy our four requirements as long as $\frac{n}{2}+h-2 \le N_{\max},T_{\max}$, hence are part of WCA~\cite{meriaux2025mobile}.

\begin{remark} 
    If $C_W = \emptyset$, then $W$ is a rectangle and so $F(\{x_d\}) = W$, which makes this a trivial case that we will not be interested in for the rest of this paper.
\end{remark}

\begin{proposition}
\label{numCov}
    Let $W$ be an orthogonal world with $n$ corners and $h$ holes in total.  Then, the number of valid corners in $W$ is exactly $\tfrac{n}{2} - 2 + h$, i.e. $|C_W| = \frac{n}{2} - 2 + h$.
    \label{prop_bound}
\end{proposition}
\begin{proof}
    For any piece $P_i$ of $W$, let $n_i$ be its number of corners and $r_i$ the number of $90^\circ$ angles in $P_i$.  The number of $270^\circ$ angles in $P_i$ is thus $n_i - r_i$, and since the sum of the interior angles of $P_i$ is $180^\circ\cdot(n_i-2)$~\cite{de2000computational}, then:
    \begin{align*}
        180(n_i-2) = 90r_i + 270(n_i-r_i),
    \end{align*}
    which is only true if $r_i = \tfrac{n_i}{2} + 2$.
    
    The number of valid corners of $W$ on $\partial P_0$ corresponds to the number of $270^\circ$ angles in $P_0$, which is therefore $n_0 - r_0 = n_0 - \left(\frac{n_0}{2} + 2\right) = \tfrac{n_0}{2} - 2$.\\
    For $i \ge 1$, the number of valid corners of $W$ on $\partial P_i$ is one less than the number of $270^\circ$ angles in $W$ on $\partial P_i$ since we exclude its top-left corner.  However, a point of $\partial P_i$ is a $270^\circ$ angle in $W$ if and only if it is a $90^\circ$ angle in $P_i$ by properties of Lebesgue measure.  This means that the number of valid corners of $W$ on $\partial P_i$ is $r_i - 1 = \tfrac{n_i}{2} + 1$.  Summing everything up, the number of valid corners in $W$ is
    \begin{align*}
        \frac{n_0}{2} - 2 + \sum_{i=1}^h \left( \frac{n_i}{2} + 1 \right) = \frac{1}{2}\sum_{i=0}^h n_i + h - 2= \frac{n}{2}+ h - 2 \,.\end{align*}\end{proof}
\setlength{\textfloatsep}{-1pt}
\setlength{\intextsep}{-1pt}

\begin{definition}[Maximal rectangle]\label{def:maximal_rectangle}
Given an orthogonal world $W \subset \mathbb{R}^2$ and a point $z \in W$, a rectangle $R \subseteq W$ containing $z$ is called \textbf{maximal} if the interiors of each of the four edges of $R$ intersect $\partial W$. Equivalently, $R$ is maximal if any strictly bigger rectangle $R' \supset R$ cannot be entirely contained in $W$. Note that for $z \in W$, there may exist many distinct maximal rectangles containing $z$.
\end{definition}

\setlength{\textfloatsep}{6pt}
\setlength{\intextsep}{6pt}
    
    \begin{figure}[ht]
    \centering
    \begin{subfigure}[b]{0.4\linewidth}
        \centering
        \begin{tikzpicture}[scale=1.2, transform shape]
            \fill[gray!35] (5.7,-2) -- (7.7,-2) -- (7.7,-2.8) -- (5.7,-2.8) -- cycle;
            \draw[thick]  (8.2,-3) -- (7.3,-3) -- (7.3,-3.5) -- (5.7,-3.5) -- (5.7,-1.8) -- (7.2,-1.8) -- (7.2,-2) -- (7.7,-2) -- (7.7,-2.5) -- (8.2,-2.5) -- (8.2,-3);
            \fill (7,-3.2) -- (7,-2.8) -- (6.6,-2.8) -- (6.6,-3) -- (6.1,-3) -- (6.1,-3.2) -- cycle;

            \fill[red] (7.2,-2) circle (0.06cm);
            \fill[red] (7.7,-2.5) circle (0.06cm);
            \fill[red] (7,-2.8) circle (0.06cm);

            \draw[-{Stealth[length=2mm, width=1.45mm]}, line width=1.2pt, draw=red](6.4,-2.3) -- (7.7,-2.3);
            \draw[-{Stealth[length=2mm, width=1.45mm]}, line width=1.2pt, draw=red](6.4,-2.3) -- (5.7,-2.3);
            \draw[-{Stealth[length=2mm, width=1.45mm]}, line width=1.2pt, draw=red](6.4,-2.3) -- (6.4,-1.98);
            \draw[-{Stealth[length=2mm, width=1.45mm]}, line width=1.2pt, draw=red](6.4,-2.3) -- (6.4,-2.8);
            \fill[blue] (6.4,-2.3) circle (0.06cm);
            \node[blue, below left] at (6.45,-2.25) {\scriptsize $z$};
            \node[gray] at (6.8,-2.5) {$R$};
        \end{tikzpicture}
        \caption{Maximal rectangle $R$ (gray) for $z \in W$ (blue); valid corners in red.}
        \label{fig:polygonal-contradiction}
    \end{subfigure}
    \hfill
    \begin{subfigure}[b]{0.48\linewidth}
    \centering
    \tikzset{
        rect_fill/.style={fill=gray!20},
        R1_edge/.style={line width=2.5pt, draw=red!80!black, dashed},
        R4_edge/.style={line width=2.5pt, draw=blue!80!black},
        R2_corner/.style={line width=1.5pt, draw=ForestGreen, fill=ForestGreen},
        R3_corner/.style={line width=2pt, draw=orange!90!black, fill=white}
    }
    \begin{tikzpicture}[scale=0.45, transform shape]
        \fill[rect_fill] (0,0) rectangle (7,3);
        \draw[R1_edge] (0,0) -- (0,3) -- (7,3);
        \draw[R4_edge] (0,0) -- (7,0) -- (7,3);
        \draw[R2_corner] (0,0) circle (5.5pt);
        \draw[R3_corner] (7,3) circle (5.5pt);
        \node[red!80!black, font=\huge\bfseries] at (-0.6, 3.2) {$R_1$};
        \node[ForestGreen, font=\huge\bfseries] at (-0.5, -0.4) {$R_2$};
        \node[orange!90!black, font=\huge\bfseries] at (7.7, 3.2) {$R_3$};
        \node[blue!80!black, font=\huge\bfseries] at (7.4, -0.4) {$R_4$};
        \node[gray] at (3.5,1.5) {\huge $R$};
    \end{tikzpicture}
    \caption{Decomposition of $\partial R$ into $R_1$ (red, dashed), $R_2$ (green, filled), $R_3$ (orange, open), $R_4$ (blue, solid).}
    \label{fig:portions_combined}
\end{subfigure}
    \caption{A maximal rectangle $R$ containing $z$ and its boundary regions.}
    \label{fig:maxrect_combined}
    \end{figure}
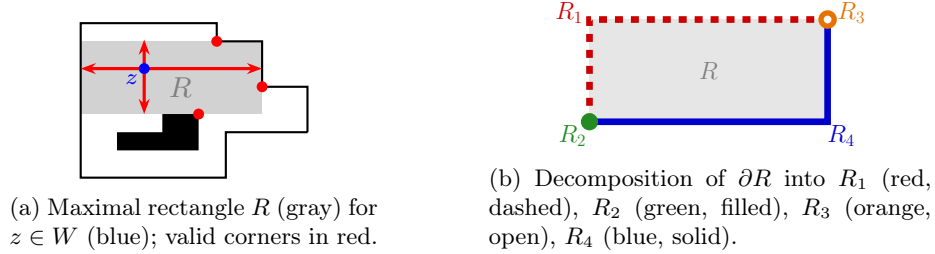

We first begin with Theorem~\ref{thm1} where we show that CADENCE is guaranteed to cover the full orthogonal world by deploying agents only to valid corners.

\begin{theorem}
\label{thm1}
Let $W$ be an orthogonal world with $C_W \neq \emptyset$. Then $F(C_W) = W$.
\end{theorem}

\begin{proof}[Proof of Theorem~\ref{thm1}]
Pick any $z \in W$ and let $R$ be a maximal rectangle containing $z$ (seen in Figure~\ref{fig:polygonal-contradiction}). Then, Lemma~\ref{lem2} gives us the existence of a valid corner $c \in \partial R$.  
Since rectangles are convex, the line segment $\overline{zc}$ is a subset of $R$, which implies that $z \in F(c) \subseteq F(C_W)$.  
Since the latter holds for any $z \in W$, we have $F(C_W) = W$.
\end{proof}
\begin{lemma}
\label{lem2}
Let $W$ be an orthogonal world such that $C_W \neq \emptyset$. Let $z \in W$ be any point, and let $R \subset W$ be a maximal rectangle containing $z$. Then, the boundary $\partial R$ contains at least one valid corner.
\end{lemma}

\begin{proof}[Proof of Lemma~\ref{lem2}]
    For this proof, we partition $\partial R$
    into four regions (see Figure~\ref{fig:portions_combined} for an illustration): $R_1$, the top
    and left edges excluding the bottom-left and top-right corners; $R_2$, consisting of the
    bottom-left corner; $R_3$, consisting of the top-right corner; and $R_4$, corresponding to the interiors of the
    bottom and right edges and the bottom-right corner.

    The Lemma follows from the three claims below. Assume for a contradiction that $\partial R$ contains no valid corner. From Claim 1, there cannot be a $270^\circ$ corner in $R_1$, otherwise it would be valid. Then, Claim 2 gives us that $R_1 \cup R_2 \cup R_3 \subset \partial W$. Consequently, by Claim 3, there is a valid corner in $R_4$, a contradiction that proves Lemma~\ref{lem2}.
    
    \vspace{5mm}
    \noindent\textbf{Claim 1.} \emph{If there exists a $270^\circ$ corner $q \in R_1$, then $q$ is valid.}
    
    \begin{proof}[Proof of Claim 1]
        Assume for a contradiction that $q$ is not valid. Since $q$ is a $270^\circ$ corner, it must be the invalid corner of some hole $P_i$, so $P_i$ must extend further down and to the right of $q$ (see Definition~\ref{def:valid}). This causes int($P_i$) to intersect with $R$, a direct contradiction of the fact that $R \subset W$ since
        $W \cap \text{int}(P_i) = \emptyset$. This means that if there exists a $270^\circ$ corner $q \in R_1$, then it must be valid.
    \end{proof}
    \vspace{5mm}
    \noindent\textbf{Claim 2.} \emph{Suppose no $270^\circ$ corner lies in $R_1$. Then, $R_1 \cup R_2 \cup R_3 \subset \partial W$.}
\begin{proof}[Proof of Claim 2]
        Assume for a contradiction that there exists some $x \in R_1 \cup R_2 \cup R_3$ that is not in $\partial W$. Then, let $S$ be the largest possible connected subset of $\partial R \setminus \partial W$ that contains $x$. One of the endpoints of $S$ (call it $m$) must intersect $\partial W$ on the interior of the (top or left) edge of $R$ that contains $m$ (from Definition~\ref{def:maximal_rectangle}). This point $m$ must be a $270^\circ$ corner, since if it was a $90^\circ$ corner, the border $\partial W$ would start intersecting int$(R)$, implying that $R \not\subseteq W$. This contradicts the assumption of this claim and finishes its proof.
\end{proof}    
    \vspace{5mm}
\noindent \textbf{Claim 3.} \textit{If $R_1 \cup R_2 \cup R_3 \subset \partial W$, then there is a valid corner in $R_4$.}
 
\begin{proof}[Proof of Claim 3]
For this proof, we analyze the three following cases.
 
\smallskip
\noindent\textbf{Case 1:} \underline{$R_4 \subset \partial W$}. Then, $R_1 \cup R_2 \cup R_3 \cup R_4 = \partial R = \partial W$, and therefore $R = W$. This implies that $C_W = \emptyset$ since $W$ is a rectangle, a contradiction.
 
\smallskip
\noindent\textbf{Case 2:} \underline{$R_4 \cap \partial W = \emptyset$}. This case directly violates Definition~\ref{def:maximal_rectangle} for $R$.
 
\smallskip
\noindent\textbf{Case 3:} \underline{$R_4 \not\subseteq \partial W$ and $R_4 \cap \partial W \ne \emptyset$}. Then, we can pick $x \in R_4 \setminus \partial W$ and let $S$ be the maximal connected subset of $R_4 \setminus \partial W$ that contains $x$. Let $m$ and $m'$ be the endpoints of $S$ (i.e. $\partial S = \{m,m'\}$), defining $m, m'$ so that $m$ is closer to $R_2$ and $m'$ closer to $R_3$ for convenience (under the standard Euclidean distance).
 We get the following four properties for $m$ and $m'$.

\begin{enumerate}
    
    \item Neither $m$ nor $m'$ is a $90^\circ$ corner. Assume $m$ is a $90^\circ$ corner for a contradiction. Note first that if $m$ is one of the bottom corners of $R$, then the interior of $S$ will intersect $\partial W$ (see Figure~\ref{fig:case3-endpoints}a), contradicting the definition of $S$. Otherwise, $\partial W$ must be penetrating the interior of $R$ at $m$ since it is a $90^\circ$ angle, contradicting the definition of $R$.
    \item One of $m$ and $m'$ is not a $270^\circ$ corner. If $m$ is a $270^\circ$ corner, then it must be the invalid corner of some hole $P_i$, meaning that $m$ must lie on the right edge of $R$ (see Definition~\ref{def:valid}). Similarly, if $m'$ is a $270^\circ$ corner, then $m'$ must lie on the bottom edge of $R$ (see Figure~\ref{fig:case3-endpoints}b). Therefore, if both $m$ and $m'$ are $270^\circ$ corners, then $m'$ is closer to $R_2$ than $m$, a contradiction.
    \item One of $m$ or $m'$ is not a $180^\circ$ angle. The only way $m$ can be a $180^\circ$ angle is if $m$ is one of the two bottom corners of $R$, otherwise $\partial W$ would intersect the interior of $R$, which is impossible. If $m'$ is a $180^\circ$ angle, then it is one of the two right corners of $R$ for the same reason. Since $m$ and $m'$ must be distinct, if both $m$ and $m'$ are $180^\circ$ angles, then the interior of at least one of the bottom or the right edges of $R$ is entirely contained in $S$ (see Figure~\ref{fig:case3-endpoints}c). Since $S$ and $\partial W$ are disjoint, this contradicts Definition~\ref{def:maximal_rectangle} for $R$.

    \item Neither $m$ nor $m'$ is a $360^\circ$ angle, since that would mean they are interior points of $W$, contradicting the fact that $m,m' \in \partial W$(see Figure~\ref{fig:case3-endpoints}d).

\end{enumerate}

Together, these four properties imply that one of $m,m'$ is a $270^\circ$ corner, while the other is a $180^\circ$ angle. This means that we are only left with analyzing the following two possibilities:
\begin{enumerate}
    \item \underline{$m$ is a $180^\circ$ angle and $m'$ is a $270^\circ$ corner}. As mentioned in the second property above, $m'$ is the invalid corner of some hole $P_i$, hence $m'$ is on the bottom edge of $R$, and so $m \in R_2$ (see Figure~\ref{fig:case3-and-split}a). Let $y'$ be the bottom right corner of $R$. Since $P_i$ does not have a valid corner on the bottom edge of $R$, then the line segment $\overline{m'y'}$ must be entirely contained in $P_i$.  Furthermore, $\partial P_i$ cannot extend above $y'$, since this would mean $m'$ was a valid corner of $P_i$.
    
    There must therefore be another piece $P_j$ of the orthogonal world $W$ that intersects the interior of the right edge of $R$ from Definition~\ref{def:maximal_rectangle} for $R$. Let $y$ be the lowest point of intersection of $P_j$ and the right edge of $R$, then $y$ must be a $270^\circ$ corner (since a $90^\circ$ corner would mean that $\partial W$ enters $R$), but also cannot be the invalid corner of $P_i$, so $y$ must be valid, a contradiction.
    
    \item \underline{$m$ is a $270^\circ$ corner and $m'$ is a $180^\circ$ angle}. This case is proven identically to the previous, except $m$ is now on the right edge of $R$ and $m' \in R_3$. Then, $y$ is the rightmost intersection of $P_j$ and the bottom edge, hence $y$ is valid.
\end{enumerate}\end{proof}

This concludes the proof of Lemma~\ref{lem2}, showing there must exist a valid corner on the boundary of $R$.
\end{proof}

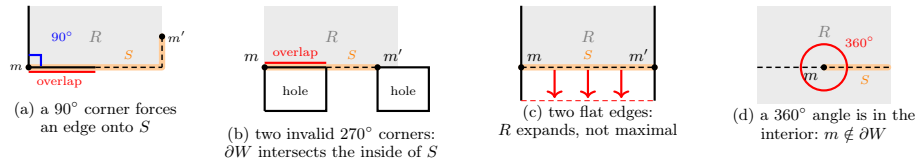
\begin{figure}[h]
\centering
\resizebox{1\linewidth}{!}{%
\begin{tikzpicture}[
  wall/.style={very thick},
  seg/.style={densely dashed, thick},
  free/.style={fill=black!8},
  bad/.style={red, very thick},
  ang/.style={blue, very thick},
  Sarc/.style={line width=3.5pt, orange!40, line cap=round, line join=round},
]
\begin{scope}[xshift=0cm]
  \fill[free] (0,0) rectangle (2.6,1.2);
  \draw[Sarc] (0,0) -- (2.6,0) -- (2.6,0.6);   
  \draw[wall] (0,0) -- (0,1.2);
  \draw[wall] (0,0) -- (1.3,0);
  \draw[seg]  (1.3,0) -- (2.6,0);
  \draw[seg]  (2.6,0.6) -- (2.6,0);
  \draw[bad]  ([yshift=-2.5pt]0,0) -- ([yshift=-2.5pt]1.3,0);
  \draw[ang]  (0,0.24) -- (0.24,0.24) -- (0.24,0.025);
  \node[orange, above] at (1.95,0.04) {\scriptsize $S$};
  \fill (0,0) circle (1.7pt) node[left]{\scriptsize $m$};
  \fill (2.6,0.6) circle (1.7pt) node[right]{\scriptsize $m'$};
  \node[blue] at (0.66,0.6) {\scriptsize $90^\circ$};
  \node[red] at (0.6,-0.28) {\scriptsize overlap};
  \node[align=center] at (1.3,-1.0) {(a) a $90^\circ$ corner forces\\[-1pt] an edge onto $S$};
  \node[gray] at (1.3,0.6) {$R$};
\end{scope}
\begin{scope}[xshift=4.6cm]
  \fill[free] (0,0) rectangle (2.6,1.2);
  \draw[Sarc] (0,0) -- (2.2,0);
  \draw[wall] (0,0) -- (0,-0.85);
  \draw[wall] (0,0) -- (1.2,0);
  \draw[wall] (1.2,0) -- (1.2,-0.85) -- (0,-0.85);
  \draw[bad]  ([yshift=2.5pt]0,0) -- ([yshift=2.5pt]1.2,0);
  \node at (0.6,-0.46) {\scriptsize hole};
  \draw[wall] (2.2,0) -- (2.2,-0.85) -- (3.2,-0.85) -- (3.2,0) -- (2.2,0);
  \node at (2.7,-0.46) {\scriptsize hole};
  \draw[seg]  (1.2,0) -- (2.2,0);
  \fill (0,0) circle (1.7pt) node[above left]{$m$};
  \fill (2.2,0) circle (1.7pt) node[above right]{$m'$};
  \node[red] at (0.6,0.28) {\scriptsize overlap};
  \node[align=center] at (1.3,-1.45)
       {(b) two invalid $270^\circ$ corners:\\[-1pt] $\partial W$ intersects the inside of $S$};
    
  \node[orange, above] at (1.7,0.04) {\scriptsize $S$};
  \node[gray] at (1.3,0.6) {$R$};
\end{scope}
\begin{scope}[xshift=9.6cm]
  \fill[free] (0,0) rectangle (2.6,1.2);
  \draw[Sarc] (0,0) -- (2.6,0);
  \draw[wall] (0,1.2) -- (0,-0.65);
  \draw[wall] (2.6,1.2) -- (2.6,-0.65);
  \draw[seg]  (0,0) -- (2.6,0);
  \draw[bad,->] (0.65,-0.05) -- (0.65,-0.6);
  \draw[bad,->] (1.3,-0.05)  -- (1.3,-0.6);
  \draw[bad,->] (1.95,-0.05) -- (1.95,-0.6);
  \draw[red,densely dashed] (0,-0.65) -- (2.6,-0.65);
  \fill (0,0) circle (1.7pt) node[above right]{$m$};
  \fill (2.6,0) circle (1.7pt) node[above left]{$m'$};
  \node[align=center] at (1.3,-1.05)
       {(c) two flat edges:\\[-1pt] $R$ expands, not maximal};
    
  \node[orange, above] at (1.3,0.04) {\scriptsize $S$};
  \node[gray] at (1.3,0.6) {$R$};
  
\end{scope}
\begin{scope}[xshift=14.2cm]
  \fill[free] (0,-0.7) rectangle (2.6,1.2);
  \draw[Sarc] (1.3,0) -- (2.55,0);
  \draw[seg] (0,0) -- (2.6,0);
  \draw[bad] (1.3,0) circle (0.45);
  \fill (1.3,0) circle (1.7pt);
  \node[below left] at (1.3,0) {$m$};
  \node[red] at (2.0,0.5) {\scriptsize $360^\circ$};
  \node[align=center] at (1.3,-1.1)
       {(d) a $360^\circ$ angle is in the\\ interior: $m\notin\partial W$};
  \node[orange, below] at (2.1,-0.04) {\scriptsize $S$};
  \node[gray] at (1.3,0.7) {$R$};
\end{scope}
\end{tikzpicture}
}
\caption{Impossible configurations for $S\subseteq R_4$ with endpoints $m$, $m'$ in Lemma~\ref{lem2}.}
\label{fig:case3-endpoints}
\end{figure}
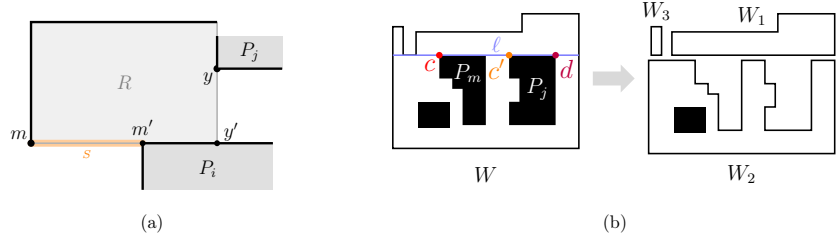
\begin{figure}[h]
\centering
\resizebox{0.9\linewidth}{!}{%
\begin{tikzpicture}[scale=0.92, transform shape,
  Redge/.style={thick, gray!60},
  edge/.style={very thick},
  good/.style={teal, very thick},
  Sarc/.style={line width=3.5pt, orange!40, line cap=round, line join=round},
]
\begin{scope}[xshift=0cm, scale=0.9]
  \fill[black!6] (0,0) rectangle (4,2.6);
  \draw[Sarc] (0,0) -- (2.4,0);
  \draw[Redge] (0,2.6) -- (0,0) -- (4,0) -- (4,2.6);
  \draw[Redge] (0,2.6) -- (4,2.6);
  \node[gray] at (2,1.3) {\large $R$};
  \fill[black!12] (2.4,0) -- (5.2,0) -- (5.2,-1.0) -- (2.4,-1.0) -- cycle;
  \draw[edge] (2.4,-1.0) -- (2.4,0) -- (5.2,0);
  \node at (3.8,-0.55) {\large $P_i$};
  \draw[edge] (0,0) -- (0,2.6) -- (4,2.6);
  \fill (0,0) circle (2.2pt) node[above left=-1pt]{\large $m$};
  \fill (2.4,0) circle (1.8pt) node[above=2pt]{\large $m'$};
  \fill (4,0) circle (1.8pt);
  \node at (4.3,0.25) {\large $y'$};
  
  \fill (4,1.6) circle (2.2pt) node[below left =-1pt]{\large $y$};
  \fill[black!12] (4,1.6) -- (5.4,1.6) -- (5.4,2.3) -- (4,2.3) -- cycle;
  \draw[edge] (4,1.6) -- (5.4,1.6);
  \draw[edge] (4,1.6) -- (4,2.3);
  \node at (4.7,1.95) {\large $P_j$};
  \node[orange, below] at (1.2,-0.04) {\scriptsize $S$};
\end{scope}
\begin{scope}[xshift=7.0cm, yshift=-0.1cm]
  \draw[thick] (0.45,2.25) -- (0.45,1.8) -- (0.2,1.8) -- (0.2,2.35) -- (0,2.35) -- (0,0) -- (3.6,0) -- (3.6,2.6) -- (2.5,2.6) -- (2.5,2.25) --  cycle;
  \fill (0.9,1.8) -- (1.8,1.8) -- (1.8,0.45) -- (1.35,0.45) -- (1.35,1.15) -- (1.15,1.15) -- (1.15,1.35) -- (0.9,1.35) -- cycle;
  \fill (2.25,1.8) -- (3.15,1.8) -- (3.15,0.45) -- (2.25,0.45) -- (2.25,0.9) -- (2.45,0.9) -- (2.45,1.35) -- (2.25,1.35) -- cycle;
  \fill (0.5,0.4) -- (1.1,0.4) -- (1.1,0.9) -- (0.5,0.9);
  \draw[thick, blue!50] (0,1.8) -- (3.6,1.8);
  \node[blue!50] at (2,2) {\large $\ell$};
  \node[white] at (1.45,1.45) {\large $P_m$};
  \node[white] at (2.8,1.15) {\large $P_j$};
  \node at (1.8,-0.5) {\large $W$};
  \fill[gray!30!white] (3.87,1.485) -- (4.41,1.485) -- (4.41,1.62) -- (4.68,1.35) -- (4.41,1.08) -- (4.41,1.215) -- (3.87,1.215) -- cycle;
  \draw[thick] (5,1.8) -- (5,2.35) -- (5.2,2.35) -- (5.2,1.8) -- cycle;
  \draw[thick] (5.4,2.25) -- (5.4,1.8) -- (8.55,1.8) -- (8.55,2.6) -- (7.45,2.6) -- (7.45,2.25) -- cycle;
  \draw[thick] (6.75,1.7) -- (6.75,0.35) -- (6.3,0.35) -- (6.3,1.05) -- (6.1,1.05) -- (6.1,1.25) -- (5.85,1.25) -- (5.85,1.7) -- (4.95,1.7) -- (4.95,-0.1) -- (8.55,-0.1) -- (8.55,1.7) -- (8.1,1.7) -- (8.1,0.35) -- (7.2,0.35) -- (7.2,0.8) -- (7.4,0.8) -- (7.4,1.25) -- (7.2,1.25) -- (7.2,1.7) -- cycle;
  \fill (5.45,0.3) -- (6.05,0.3) -- (6.05,0.8) -- (5.45,0.8);
  \node at (6.95,2.55) {\large $W_1$};
  \node at (6.75,-0.5) {\large $W_2$};
  \node at (5.1,2.65) {\large $W_3$};
  \fill[red] (0.9,1.8) circle (1.8pt) node[below left =-1pt]{\Large $c$};
  \fill[orange] (2.25,1.8) circle (1.8pt) node[below left =-1pt]{\Large $c'$};
  \fill[purple] (3.15,1.8) circle (1.8pt) node[below right =-1pt]{\Large $d$};
\end{scope}
\node at (2.34,-1.55) {\normalsize (a)};
\node at (11.3,-1.55) {\normalsize (b)};
\end{tikzpicture}%
}
\caption{(a) Case~3 showing why a valid $270^\circ$ corner $y$ on $\partial R$ must exist.\\ (b) Theorem~2: splitting $W$ along $\ell$.}
\label{fig:case3-and-split}
\end{figure}

\begin{theorem}
    \label{main2}
        Let $W$ be an orthogonal world such that $C_W \ne \emptyset$. Then, $G(C_W)$ is connected.
\end{theorem}

\newpage

\begin{proof}[Proof of Theorem~\ref{main2}]
    We proceed by induction on the number $h$ of holes of $W$.

    \noindent \textbf{Base case:} If $h=0$, there is no invalid corner in $W$ since there is no hole, so $C_W$ is exactly the set of all $270^\circ$ angles. So, from Remark~\ref{vg}, $G(C_W)$ is connected.

    \noindent \textbf{Induction step:} If $h > 0$, then let $P_m$ be the hole of $W$ whose valid corner $c$ has the highest vertical coordinate, then lowest horizontal coordinate in case of equality. Consider now the horizontal line segment $\ell \subseteq W$ of maximal length such that $\ell$ contains $c$. $W\setminus\ell$ thus consists of $k \ge 2$ distinct pairwise-disconnected shapes $W'_1,W'_2,\ldots,W'_k$ (notice that $k$ is at least $2$ since the endpoints of $\ell$ must both be on vertical edges of $P_0$, disconnecting $W$ into at least two sets). Then, each $W_i = \mathrm{cl}(W'_i)$ is an orthogonal world with strictly less than $h$ holes, since $P_m$ is not a hole of any $W_i$. An example is seen in Figure~\ref{fig:case3-and-split}b where $k=3$.

    For any point $c' \in W_i \setminus \ell$, we claim that $c'$ is valid in $W_i$ if and only if $c'$ is valid in $W$. To see this, we first note that $c'$ has the same angle in $W_i$ and $W$, hence we may assume that this angle is $270^\circ$. On one hand, if $c'$ is the invalid corner of some hole $P_j$ in $W$, then $P_j \cap \ell = \emptyset$ since $c' \notin \ell$, so $P_j$ is also a hole of $W_i$, hence $c'$ is invalid in $W_i$. On the other hand, if $c'$ is the invalid corner of a hole $P_j$ in $W_i$, then $P_j$ is also a hole in $W$, so $c'$ is invalid in $W$, proving our claim. Letting $C = C_W \cap \ell$, this gives us that $C_W = C_{W_1} \cup \ldots \cup C_{W_k} \cup C$.
    
    Finally, our induction hypothesis gives us that each non-empty $G(C_{W_i})$ is connected. Since these are all subgraphs of $G(C_W)$ (an edge in $G(C_{W_i})$ is also an edge in $G(C_W)$) and $C$ is connected ($\ell$ is a segment therefore $C$ is a clique), then it only remains to show that there is an edge from $C$ to all of the non-empty $C_{W_i}$'s, implying that $G(C_W)$ is connected. 
    
    Fix any $W_i$; there must be a point $c' \in C \cap W_i$ that is a $270^\circ$ corner in $W$, since otherwise $\ell$ could not have disconnected $W_i'$ from the others. Furthermore, if $c'$ is the invalid corner of some hole $P_j$ in $W$, then the $270^\circ$ corner $d$ on the right of $c'$ in $P_j$ is valid in $W$ and is on the border of $W_i$ (see Figure~\ref{fig:case3-and-split}b, where $c'$ and $d$ are in both $W_1$ and $W_2$). In all cases, for each $W_i$, there is a valid corner $d \in C \cap W_i$. Using Theorem \ref{thm1}, if $C_{W_i} \ne \emptyset$, then $d$ is seen by some valid corner of $W_i$, meaning there is an edge from $C_{W_i}$ to $C$ in $G(C_W)$, finishing the proof.
\end{proof}

\section{Empirical Evaluation}
\label{sec:experiments}
We now evaluate five selection rules for $x$ (see line 4 of Alg.~\ref{cadence_algo}) as distinct members of the CADENCE subclass of WCA. Results are then compared to the Incremental Self-Deployment algorithm (ISDA)~\cite{howard2002incremental}, the baseline.

\textbf{Setup.} We use randomly generated discrete orthogonal worlds (grid worlds) at three sizes (as in~\cite{meriaux2025mobile}),
$50\times50$, $100\times100$, and $250\times250$, each with holes and a fixed deployment point $x_d$.
In total, $300$ orthogonal worlds were generated, each repeated with $5$ seeds (each world being run 5 times by each algorithm). Each run is capped by an agent budget
$N_{\max}=\tfrac{n}{2}+h-2$ (the worst-case CADENCE bound) and a step budget $T_{\max}$, set to
$5{,}000$, $10{,}000$, and $30{,}000$ (as in~\cite{meriaux2025mobile}) selection steps for the three resolutions respectively; a run
that reaches either cap before full coverage is recorded as non-converged.

\textbf{Process.} To simulate real-life MANETs, we run CADENCE by deploying agents from $x_d$, they walk one square at a time (step by step) towards their allocated valid corner, and the total number of steps they take is denoted "Steps" (see Table~\ref{tab:results_combined}). Throughout this process, we call an agent \textit{de-allocated}~\cite{meriaux2025mobile} if its removal would not affect coverage or connectivity. Then, we may re-allocate it in subsequent iterations instead of deploying a new agent. Once $W$ is finally covered (i.e. $F(\mathcal{A}) = W$), the nb. of agents in $W$ is denoted "Max", and this number minus the nb. of de-allocated agents is denoted "Final" (see Table~\ref{tab:results_combined}).

\textbf{Selection rules.} \emph{rand\_point} picks a uniformly random visible valid corner. \emph{max\_dist} picks the visible valid corner that maximizes the minimal distance to the deployment point and to any de-allocated agent, while \emph{min\_dist} minimizes that minimal distance. \emph{most\_edge} picks the visible valid corner that sees the most frontier cells (the visible cells with a non-visible neighbor), while \emph{least\_edge} picks the one that sees the least. On the other hand, ISDA~\cite{howard2002incremental} could deploy agents anywhere in the current field of view, not only at corners. There are two variants: \emph{ISDA edge} randomly selects a frontier cell, while \emph{ISDA any} randomly selects any visible cell.

\begin{table}[h]
\centering
\caption{Mean final nb. of agents (Final), maximal nb. of agents (Max), and total nb. of steps (Steps) over random orthogonal worlds (lower is better; best per column in bold).
\na\, indicating less than 80\% of the runs finished within the step budget for $100^2$, and 0\% of the runs finished in under 24 hours for $250^2$.}
\label{tab:results_combined}
\setlength{\tabcolsep}{3pt}

\resizebox{1\columnwidth}{!}{%
\begin{tabular}{lcccccccccccc}
\hline
 & \multicolumn{6}{c}{$50^2$} & \multicolumn{6}{c}{$100^2$} \\
\cline{2-7}\cline{8-13}
Method & Final & $\sigma$ & Max & $\sigma$ & Steps & $\sigma$ & Final & $\sigma$ & Max & $\sigma$ & Steps & $\sigma$ \\
\hline
\multicolumn{13}{l}{\textit{CADENCE}} \\
\hline
least\_edge & $13.4$ & $3.1$ & $20.0$ & $3.1$ & $1218$ & $337$ & $26.5$ & $3.5$ & $38.9$ & $4.1$ & $6168$ & $1037$ \\
rand\_point & $13.1$ & $3.2$ & $22.2$ & $3.0$ & $1488$ & $412$ & $25.7$ & $3.6$ & $42.3$ & $4.3$ & $8320$ & $1345$ \\
max\_dist & $13.0$ & $3.1$ & $21.9$ & $3.9$ & $1202$ & $411$ & $25.6$ & $3.5$ & $41.1$ & $5.0$ & $6133$ & $1658$ \\
most\_edge & $\mathbf{13.0}$ & $3.1$ & $21.1$ & $3.3$ & $1558$ & $467$ & $\mathbf{25.3}$ & $3.5$ & $40.3$ & $3.9$ & $8609$ & $1582$ \\
min\_dist & $13.6$ & $3.2$ & $\mathbf{16.2}$ & $3.3$ & $705$ & $262$ & $26.8$ & $3.6$ & $\mathbf{30.4}$ & $3.8$ & $2757$ & $705$ \\
\hline
\multicolumn{13}{l}{\textit{ISDA baseline}} \\
\hline
ISDA edge & $16.1$ & $3.3$ & $16.3$ & $3.3$ & $\mathbf{573}$ & $247$ & $33.4$ & $4.0$ & $33.6$ & $4.1$ & $\mathbf{2547}$ & $800$ \\
ISDA any & $13.5$ & $3.0$ & $30.2$ & $7.8$ & $6120$ & $3945$ & \na & \na & \na & \na & \na & \na \\
\hline
\end{tabular}%
}

\vspace{1em} 

\begin{tabular}{lcccccc}
\hline
 & \multicolumn{6}{c}{$250^2$} \\
\cline{2-7}
Method & Final & $\sigma$ & Max & $\sigma$ & Steps & $\sigma$ \\
\hline
\multicolumn{7}{l}{\textit{CADENCE}} \\
\hline
least\_edge & $47.4$ & $6.4$ & $76.8$ & $8.6$ & $36003$ & $5295$ \\
rand\_point & $44.6$ & $6.3$ & $74.2$ & $6.9$ & $51434$ & $6576$ \\
max\_dist & $44.9$ & $6.2$ & $71.1$ & $8.1$ & $32104$ & $6790$ \\
most\_edge & $\mathbf{43.8}$ & $6.2$ & $69.7$ & $7.1$ & $52060$ & $6792$ \\
min\_dist & $47.6$ & $6.2$ & $\mathbf{53.5}$ & $6.0$ & $\mathbf{11514}$ & $2141$ \\
\hline
\multicolumn{7}{l}{\textit{ISDA baseline}} \\
\hline
ISDA edge & $62.9$ & $6.8$ & $63.1$ & $6.8$ & $11561$ & $2691$ \\
ISDA any & \na & \na & \na & \na & \na & \na \\
\hline
\end{tabular}
\end{table}
\setlength{\textfloatsep}{5pt}
\setlength{\intextsep}{5pt}


\textbf{Results.} The CADENCE variants trade off along minimizing the aforementioned metrics Final, Max, and Steps, and no single selection rule wins on all of them. See Table~\ref{tab:results_combined} for full details of the results.

First, \emph{ISDA any} does worse than \emph{least\_edge} on every metric, and \emph{least\_edge} is never the best performing rule on any metric. This is what we expected, since these rules do not pick points according to any minimizing strategy.

Next, \emph{min\_dist} and \emph{ISDA edge} have similar performances on Steps and Max and are the two best selection rules for these two metrics, but \emph{min\_dist} overtakes \emph{ISDA edge} on every metric as the world size grows. Indeed, for $50^2$, $100^2$, and $250^2$ sized worlds, \emph{min\_dist} performs better for Final by $15.5\%$, $19.8\%$, and $24.3\%$ and for Max by $0.6\%$, $9.5\%$, and $15.2\%$, respectively. Meanwhile, \emph{ISDA edge} performs better on Steps by $18.7\%$, $7.6\%$, and $-0.4\%$, respectively, meaning that \emph{min\_dist} eventually overtakes \emph{ISDA edge} on every metric for bigger worlds. This might suggest that deploying agents at the corners becomes increasingly valuable as the worlds' sizes grow.

Finally, \emph{most\_edge}, \emph{max\_dist}, and \emph{rand\_point} have very similar performances on every metric, and are the best at minimizing Final. These results were surprising to us, since such rules tend to try expanding the field of view of the agents as fast as possible, therefore leaving less room for de-allocation. However, we observed that it is only towards the very end of the process that these rules start de-allocating agents, explaining why they do worse on the Max metric since they accumulate more agents overall, but then do better on the Final metric since they have lots of agents to de-allocate.

Overall, it is clear that the CADENCE selection rules perform better than the ISDA rules, except for the Steps metric on smaller sized worlds.

\section{Conclusion}
\label{sec:conclusion}
This paper studies the WCA that solve the POCGAGP. We give a full proof that CADENCE, a subclass of WCA, achieves full coverage of orthogonal worlds with at most $\tfrac{n}{2}+h-2$ agents despite the partial observability constraint, matching the best known bound for the fully observable setting. Our empirical evaluation of deployment-order heuristics shows that the strongest CADENCE variants outperform the ISDA baseline on every metric as world sizes grow. 
%
\bibliographystyle{splncs04}
\bibliography{references}

\end{document}